\documentclass[letterpaper,11pt]{article}

\usepackage[margin=1in]{geometry}
\usepackage{color}
\usepackage{tikz}
\usepackage{authblk}

\usepackage{amsmath}
\usepackage{amssymb}
\usepackage{amsthm}
\usepackage{amsfonts}
\usepackage{xcolor}
\usepackage{upgreek}
\usepackage{enumitem}
\usepackage{dsfont}
\usepackage{tcolorbox}
\usepackage[utf8]{inputenc}
\usepackage{hyperref}
\usepackage{multirow}

\newtheorem{theorem}{Theorem}
\newtheorem{lemma}{Lemma}[section]

\newtheorem{definition}{Definition}[section]

\newcommand{\bcomment}[1]{}
\newcommand{\TODO}[1]{}
\newcommand{\define}{\triangleq}

\newcommand{\eql}[1]{\overset{(#1)}{=}}

\newcommand{\pth}[1]{\left(#1\right)} 
\newcommand{\croc}[1]{\left[#1\right]} 
\newcommand{\acc}[1]{\left\{#1\right\}} 
\newcommand{\norm}[1]{\left|#1\right|} 

\newcommand{\bpth}[1]{\big(#1\big)}

\newcommand{\Bpth}[1]{\Big(#1\Big)}

\newcommand{\pthMat}[4]{\begin{pmatrix} #1 & #2 \\ #3 & #4\end{pmatrix}}

\newcommand{\crocMat}[4]{\begin{bmatrix} #1 & #2 \\ #3 & #4\end{bmatrix}}
\newcommand{\crocVec}[2]{\begin{bmatrix} #1 \\ #2\end{bmatrix}}

\newcommand{\natS}{\mathbb{N}}

\newcommand{\realS}{\mathbb{R}}
\newcommand{\compS}{\mathbb{C}}

\newcommand{\E}[1]{\mathbb{E}\left[#1\right]}
\newcommand{\ind}[1]{\mathds{1}\left\{#1\right\}}

\newcommand\independent{\protect\mathpalette{\protect\independenT}{\perp}}
\def\independenT#1#2{\mathrel{\rlap{$#1#2$}\mkern2mu{#1#2}}}

\newcommand{\calC}{\mathcal{C}}

\newcommand{\calE}{\mathcal{E}}

\newcommand{\calM}{\mathcal{M}}
\newcommand{\calN}{\mathcal{N}}
\newcommand{\calO}{\mathcal{O}}

\newcommand{\calS}{\mathcal{S}}
\newcommand{\calT}{\mathcal{T}}
\newcommand{\calU}{\mathcal{U}}
\newcommand{\calV}{\mathcal{V}}

\DeclareMathOperator*{\argmax}{argmax}

\newcommand{\ul}[1]{\vec{#1}}
\newcommand{\dul}[1]{\mathbf{#1}}
\newcommand{\diag}{\operatorname{diag}}
\newcommand{\tr}{\operatorname{tr}}
\newcommand{\Var}{\operatorname{Var}}

\newcommand{\eq}{\mathrel{\phantom{=}}}

\newcommand{\UA}{\calU}
\newcommand{\UB}{\calV}
\newcommand{\Siga}{\dul{\Sigma}_{\textrm{a}}}
\newcommand{\Sigb}{\dul{\Sigma}_{\textrm{b}}}
\newcommand{\Sigab}{\dul{\Sigma}_{\textrm{ab}}}
\newcommand{\La}{\dul{L}_{\textrm{a}}}
\newcommand{\Lb}{\dul{L}_{\textrm{b}}}
\newcommand{\Ta}{T_a}
\newcommand{\Tb}{T_b}
\newcommand{\mua}{\ul{\mu}_a}
\newcommand{\mub}{\ul{\mu}_b}
\newcommand{\da}{d_a}
\newcommand{\db}{d_b}

\begin{document}
\title{Database Alignment with Gaussian Features}
\author[1]{Osman Emre Dai\thanks{oedai@gatech.edu}}
\author[2]{Daniel Cullina\thanks{dcullina@princeton.edu}}
\author[1]{Negar Kiyavash\thanks{negar.kiyavash@ece.gatech.edu}}
\affil[1]{Georgia Tech, Department of Industrial \& Systems Engineering}
\affil[2]{Princeton University, Department of Electrical Engineering}
\affil[3]{Georgia Tech, Department of Electrical and Computer Engineering and Department of Industrial \& Systems Engineering}
\date{}
\maketitle

\begin{abstract}
  We consider the problem of aligning a pair of databases with jointly Gaussian features. We consider two algorithms, complete database alignment via MAP estimation among all possible database alignments, and partial alignment via a thresholding approach of log likelihood ratios. We derive conditions on mutual information between feature pairs, identifying the regimes where the algorithms are guaranteed to perform reliably and those where they cannot be expected to succeed.
\end{abstract}

\section{Introduction}

Consider the following setting: There are a large set of entities (e,g, users) with some measurable characteristics. Let the measures of these characteristics be jointly Gaussian, with known statistics. We refer to these measures as features. Consider two different sources, each providing a database with lists of features for these entities. Furthermore, let one these sources lack proper labeling for features that would allow for the identification of feature pairs from the two sources that correspond to the same entity. This might be due to privacy concerns, if the mentioned features provided by the source contain sensitive information that ought to remain anonymous, or it might simply be that a reliable labeling is not available.

If the correlation between features pairs is sufficiently strong, then it is possible to exploit this correlation to identify correspondences between the two databases and in fact generate a perfect alignment between the feature lists. Such a capability might be a valuable tool to recuperate missing information by labeling unlabeled features or by allowing the junction of measurements coming from distinct sources. However it also has serious implications in privacy as it makes anonymous data vulnerable to deanonymization attacks \cite{imdb}.

It then becomes critical to understand the limitations of database alignment and to identify the conditions that characterize these limitations. This allows us to assess the feasibility and reliability of alignment procedures as well as the vulnerability of deanonymization schemes.  In this study we investigate the conditions that guarantee either the achievability of alignment or its infeasibility. We analyze these conditions for both partial alignments and as well as for complete alignments. Cullina et al. have recently analyzed this problem for the case of discrete random variables, introducing a new correlation measure characterizing the feasibility of alignment \cite{DBLP:conf/isit/2018}. Takbiri et al. have investigated a related problem where the feature of each user is Gaussian with characteristic statistics and has correlation with other user features \cite{takbiriGaussian}. In this setting an adversary with perfect knowledge of system statistics attempts to match features with the characteristic user statistics. This follows the authors' previous studies of the same setting for discrete valued features and with data obfuscation \cite{takbiri1},\cite{DBLP:journals/corr/abs-1805-01296}.

The database alignment problem is connected to the widely studied graph alignment problem.
In that setting, each feature is associated with a pair of anonymized users.
In the simplest case, the feature is a Bernoulli random variable indicating the presence or absence of an edge between the users.
A recent line of work has characterized the information theoretic limits of the graph alignment problem \cite{pedarsani_privacy_2011,cullina_exact_2017,cullina_improved_2016}.
The problem of aligning correlated Wigner matrices, in which each feature is a Gaussian random variable, has served as a proxy for understanding the effectiveness of graph alignment algorithms \cite{ding_efficient_2018}.


\section{Model}

\paragraph{Notation}

We denote random variables by capital letters and fixed values by lowercase letters.
For a set \(\calS\) and finite sets \(\calT,\calU\), we denote by \(\calS^{\calT\times\calU}\) the set of matrices with entries in \(\calS\), rows indexed by \(\calT\) and columns indexed by \(\calU\).
We mark vectors with arrows and write matrices in boldface.
Given some matrix \(\dul{z}\),  we denote its \(i\)-th row by \(\dul{z}_{i*}\), \(j\)-th column by \(\dul{z}_{*j}\) and its \((i,j)\)-th entry by \(\dul{z}_{ij}\).
We denote the identity matrix with rows and columns indexed by $\calT$ by \(\dul{I}^{\calT}\).
When the indexing set is clear from context, we will drop the superscript.
We denote the set of integers from \(1\) to \(n\) by \([n]\).

\subsection{General problem formulation}
\label{subsec:problem}

In this model, a database is just a function from a set of users to some space.
The value of the function for a user $u$ is the database entry for that user.
Cullina, Mittal, and Kiyavash considered database entries in finite alphabets\cite{DBLP:conf/isit/2018}.
In this paper, we consider database entries that are finite dimensional real vectors sampled from a gaussian distribution.

We are given two sets of user identifiers, \(\UA\) and \(\UB\), with \(|\UA|=|\UB|=n\).
We express the content of databases by matrices \(\dul{A}\in\realS^{\UA\times[\da]}\) and \(\dul{B}\in\realS^{\UB\times[\db]}\), so \(\da\) and \(\db\) are the lengths of feature vectors.

There exists a natural bijective correspondence between the identifier sets, i.e. each identifier in one set is related to exactly one identifier in the other set. We express this correspondence by the bijective matching \(M \subseteq \UA\times\UB\).


Let $p_{\ul{X}\ul{Y}}$ be the density of jointly gaussian random variables $\ul{X} \in \realS^{\da}$ and $\ul{Y} \in \realS^{\db}$ such that $(\ul{X},\ul{Y}) \sim \calN\bpth{\ul{\mu},\dul{\Sigma}}$.

The density $p_{\dul{A}\dul{B}|M}$ is defined as follows.
For each $(u,v) \in M$, $(\dul{A}_{u*},\dul{B}_{v*}) \sim p_{\ul{X}\ul{Y}}$ and these $n$ random variables are independent:
\[
  p_{\dul{A}\dul{B}|M}(\dul{a},\dul{b}|m) = \prod_{(u,v) \in m} p_{\ul{X}\ul{Y}}(\dul{a}_{u*},\dul{b}_{v*}).
\]
The matching $M$ is uniformly distributed over the $n!$ bijective matchings between $\UA$ and $\UB$.

The database alignment problem is to recover $M$ from $\dul{A}\dul{B}$, given knowledge of $p_{\ul{X}\ul{Y}}$.

Observe that the rows of $\dul{A}$ are i.i.d. and that $\dul{A}$ is independent of $M$.
The same is true for $\dul{B}$.
In other words, by examining one database, an observer learns nothing about $M$.

A pair of databases are illustrated in Figure \ref{fig}.

\paragraph{Canonical form of covariance}
We write $\ul{\mu} = \crocVec{\mua}{\mub}$ and \(\dul{\Sigma} = \crocMat{\Siga}{\Sigab}{\Sigab^\top}{\Sigb}\), so $\dul{A}_{u*} \sim \calN(\mua,\Siga)$ for each $u \in \UA$ and $\dul{B}_{u*} \sim \calN(\mub,\Sigb)$ for each $v \in \UB$.

Let $\da'$ be the dimension of the support of $\ul{X}$, i.e. the rank of $\Siga$.
Let $\phi : \realS^{[\da]} \to \realS^{\da'}$ be an affine transformation that is injective on the support of $\ul{X}$.
If we apply $\phi$ to each row of $\dul{A}$, which can be done without knowledge of $M$, we obtain an equivalent database alignment problem.
Similarly, the database $\dul{B}$ can be transformed to obtain an equivalent problem.

For any gaussian database alignment problem, there is an equivalent problem with $\ul{\mu} = \ul{0}$ and
\[
  \dul{\Sigma} = \crocMat{\dul{I}^{[d]}}{\diag(\ul{\rho})}{\diag(\ul{\rho})}{\dul{I}^{[d]}} = \bigoplus_{i \in [d]} \crocMat{1}{\rho_i}{\rho_i}{1}
\]
where $d = \min(d_a,d_b)$.
Thus the correlation structure of $(\ul{X},\ul{Y})$ is completely summarized by the vector $\ul{\rho} \in \realS^d$.
The explicit transformations that put $\dul{\Sigma}$ into this form are described in \hyperref[sec:featureTransformation]{Appendix \ref*{sec:featureTransformation}}.

\bcomment{

  We say two identifiers are not related only if their feature vectors are independent, i.e. \((u,v)\notin M \implies \ul{A}_{u*}\independent \ul{B}_{v*}\).
  Given any pair of related identifiers \((u,v)\in M\), the pair of feature vectors \((\ul{A}_{u*},\ul{B}_{v*})\) is jointly Gaussian and distributed as \(\calN\bpth{\ul{\mu},\dul{\Sigma}}\) where \(\ul{\mu}\) and \(\dul{\Sigma}\) do not depend on the pair \((u,v)\in M\).
  We write \(\dul{\Sigma} = \pthMat{\Siga}{\Sigab}{\Sigab^\top}{\Sigb}\) and assume both \(\Siga\in\realS^{d_1\times d_1}\) and \(\Sigb\in\realS^{d_2\times d_2}\) are invertible.

  As shown in Appendix \ref{sec:featureTransformation}, it is possible to transform the feature vectors to get a special case of the problem setting with some nice properties. Therefore assuming these properties does not result in a loss of generality.

  Let feature vectors in the two datasets have equal length, i.e. \(d_1=d_2=d\). Furthermore let features have zero mean and unit variance, i.e. \(\ul{\mu}=\ul{0}\) and \(\Siga=\Sigb=\dul{I}^d\). Finally let features pairs at each index be independent, i.e. let \(\Sigma_{12}=\diag(\ul{\rho})\) be a diagonal matrix with \(-1<\rho_i<1\) for all \(i\in[d]\).
  
}

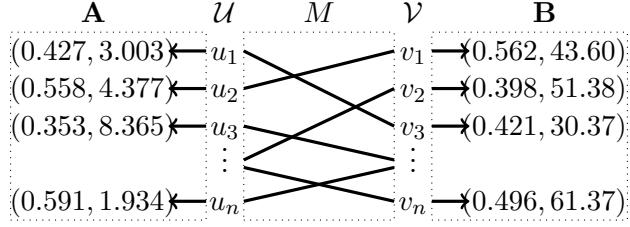
\begin{figure}
  \centering
  \begin{tikzpicture}[
    scale = 1/2,
    text height = 1.5ex,
    text depth =.1ex,
    b/.style={very thick}
    ]
    \draw (0,7) node {$M$};
    \draw (-6,7) node {$\dul{A}$};
    \draw (6,7) node {$\dul{B}$};
    \draw[dotted] (-2,6.5) -- (2,6.5) -- (2,1.5) -- (-2,1.5) -- cycle;
    \draw[dotted] (-8.2,6.5) -- (-3,6.5) -- (-3,1.5) -- (-8.2,1.5) -- cycle;
    \draw[dotted] ( 8.2,6.5) -- ( 3,6.5) -- ( 3,1.5) -- ( 8.2,1.5) -- cycle; 
    \draw (-2.5,7) node {$\UA$};
    \draw ( 2.5,7) node {$\UB$};
    \draw (-2.5,6) node {$u_1$};
    \draw (-2.5,5) node {$u_2$};
    \draw (-2.5,4) node {$u_3$};
    \draw (-2.5,3) node {$\vdots$};
    \draw (-2.5,2) node {$u_n$};

    \draw (2.5,6) node {$v_1$};
    \draw (2.5,5) node {$v_2$};
    \draw (2.5,4) node {$v_3$};
    \draw (2.5,3) node {$\vdots$};
    \draw (2.5,2) node {$v_n$};

    \draw[b] (-2,6) -- (2,4);
    \draw[b] (-2,5) -- (2,6);
    \draw[b] (-2,4) -- (2,3.1);
    \draw[b] (-2,3.1) -- (2,5);
    \draw[b] (-2,2.9) -- (2,2);
    \draw[b] (-2,2) -- (2,2.9);

    \draw (-6,6) node {$(0.427,3.003)$};
    \draw (-6,5) node {$(0.558,4.377)$};
    \draw (-6,4) node {$(0.353,8.365)$};
    \draw (-6,2) node {$(0.591,1.934)$};

    \draw (6,6) node {$(0.562,43.60)$};
    \draw (6,5) node {$(0.398,51.38)$};
    \draw (6,4) node {$(0.421,30.37)$};
    \draw (6,2) node {$(0.496,61.37)$};

    \draw[b,->] (-3,6) -- (-4,6);
    \draw[b,->] (-3,5) -- (-4,5);
    \draw[b,->] (-3,4) -- (-4,4);
    \draw[b,->] (-3,2) -- (-4,2);

    \draw[b,->] (3,6) -- (4,6);
    \draw[b,->] (3,5) -- (4,5);
    \draw[b,->] (3,4) -- (4,4);
    \draw[b,->] (3,2) -- (4,2);

  \end{tikzpicture}
  \caption{
    Databases $\dul{A}$ and $\dul{B}$ with $d_a = d_b = 2$ and a matching $M$ between their user identifier sets.
  }
  \label{fig}
\end{figure}

\subsection{Correlation measures}
\label{section:corr}
Let \(I_{\ul{X}\ul{Y}} \define I(\dul{A}_{u*},\dul{B}_{v*},|(u,v)\in M)\) denote the mutual information between any pair of related identifiers coming from \((u,v)\in M\).
Then
\begin{align*}
  I_{\ul{X}\ul{Y}}
  &= -\frac{1}{2}\log\frac{\det\bpth{\dul{\Sigma}}}{\det\bpth{\Siga}\cdot\det\bpth{\Sigb}}\\
  &= -\frac{1}{2}\sum_{i\in[d]}\log\pth{1-\rho_i^2}.
\end{align*}
Under the canonical formulation where \(\Siga=\Sigb=\dul{I}^d\) and \(\Sigab=\diag(\ul{\rho})\) this becomes

Given any \((u,v)\in M\) and \((\ul{X},\ul{Y})=(\dul{A}_{u*}^\top,\dul{B}_{v*}^\top)\),
\begin{align*}
  \sigma_{\ul{X}\ul{Y}}^2 \define \Var\pth{\log \frac{p_{\ul{X}\ul{Y}}(\ul{X},\ul{Y})}{p_{\ul{X}}(\ul{X})p_{\ul{Y}}(\ul{Y})}}.
\end{align*}
Then \(\sigma_{\ul{X}\ul{Y}}^2 = \tr\bpth{\Siga^{-1}\Sigab\Sigb^{-1}\Sigab^\top}\).
Furthermore under the canonical formulation where \(\Siga=\Sigb=\dul{I}^d\) and \(\Sigab=\diag(\ul{\rho})\) this simplifies to \(\sigma_{\ul{X}\ul{Y}}^2 = \sum \rho_i^2\).
These calculations are made explicit in our supplementary material.

Note that \(\sigma_{\ul{X}\ul{Y}}\) is upper bounded by \(\sqrt{2I_{\ul{X}\ul{Y}}}\). This can easily be seen in the canonical formulation, where \(\sigma_{\ul{X}\ul{Y}}^2 = \sum \rho_i^2 \leq -\sum\log(1-\rho_i^2) = 2 I_{\ul{X}\ul{Y}}\).


\section{Results}

Our results identify conditions on \(I_{\ul{X}\ul{Y}}\) and \(\sigma_{\ul{X}\ul{Y}}\), as defined in Section~\ref{section:corr}.

\paragraph{MAP estimation}

The algorithm considers all possible alignments between the two sets and chooses the most likely one.
The log likelihood of an alignment is, by the independence of correlated feature pairs, equal to the sum of the log likelihood of each aligned feature pair. MAP estimation can then be implemented by computing the joint likelihood for each feature pair in \(\calO(n^2d)\)-time and computing the maximum weight matching between databases in \(\calO(n^3)\)-time using the Hungarian algorithm.

\newcommand{\thmAchMAP}{
    \textbf{\upshape(Achievability)} If mutual information between feature pairs \(I_{\ul{X}\ul{Y}} \geq 2\log n + \omega(1)\), then the MAP estimator returns the proper alignment with probability \(1-o(1)\).
}
\begin{theorem}
\label{thm:achMAP}
\thmAchMAP
\end{theorem}

\newcommand{\thmConMAP}{
    \textbf{\upshape(Converse)} Let \(d\in\natS\) such that \(d\geq\omega(1)\). Furthermore let \(\Siga=\Sigb=\dul{I}^d\) and \(\Sigab=\rho \dul{I}\). If \(I_{\ul{X}\ul{Y}} \leq 2\log n(1 - \Omega(1))\), then any for algorithm, the probability of returning the proper alignment is \(o(1)\).
}
\begin{theorem}
\label{thm:conMAP}
\thmConMAP
\end{theorem}

\paragraph{Binary hypothesis testing}

The algorithm checks every possible pair of identifiers and uses a threshold-based method to decide whether to match the pair or not.
This can be done in \(\calO(n^2d)\)-time, which is the complexity of computing joint likelihoods for each feature pair.

\newcommand{\thmAchLRT}{
    \textbf{\upshape(Achievability)} If
    \begin{align*}
        I_{\ul{X}\ul{Y}} \geq \sigma_{\ul{X}\ul{Y}}\cdot\sqrt{\frac{n}{\varepsilon_{FN}}} + \log\frac{n^2}{\varepsilon_{FP}},
    \end{align*}
    then, choosing the threshold such that \(\log(n^2/\varepsilon_{FP})\leq \tau \leq I_{\ul{X}\ul{Y}}-\sigma_{\ul{X}\ul{Y}}\sqrt{n/\varepsilon_{FN}}\), the binary hypothesis test gives no more than \(\varepsilon_{FN}\) false negatives and \(\varepsilon_{FP}\) false positives in expectation.
}

\begin{theorem}
\label{thm:achLRT}
\thmAchLRT
\end{theorem}

It follows that the following regimes are achievable:
\begin{align*}
    \bullet\,\, I_{\ul{X}\ul{Y}}&\geq \log(n)+\omega(1) & \varepsilon_{FN}&\leq o(n) & \varepsilon_{FP} &\leq o(n)\\
    \bullet\,\, I_{\ul{X}\ul{Y}}&\geq 2\log(n)+\omega(1) & \varepsilon_{FN}&\leq o(n) & \varepsilon_{FP} &\leq o(1)\
\end{align*}

The next theorem holds for databases with any distribution of feature pairs, i.e. not only Gaussians.

\newcommand{\thmConLRT}{
    \textbf{\upshape(Converse)} For any binary hypothesis test, the expected number of false negatives \(\varepsilon_{FN}\) and false positives \(\varepsilon_{FP}\) is lower bounded as
    \begin{align*}
        \varepsilon_{FN}+\varepsilon_{FP} \geq \frac{n}{2}\pth{1 - \frac{I_{\ul{X}\ul{Y}}}{\log n}}\pth{1-\mathcal{O}\pth{\frac{1}{\log n}}}.
    \end{align*}
}
\begin{theorem}
\label{thm:conLRT}
\thmConLRT
\end{theorem}

It follows that, if \(I_{\ul{X}\ul{Y}} \leq \log n\bpth{1 - \Omega(1)}\), then any binary hypothesis test has expected number of errors \(\varepsilon_{FN}+\varepsilon_{FP} \geq \Omega(n)\).

\TODO{
\paragraph{Summary}

In the table, \(\varepsilon^{-}\) and \(\varepsilon^{+}\) denote the expected number of false negatives and false positives, respectively. The MAP algorithm gives us a size \(n\) matching, so for every false negative we get a false positive and thus \(\varepsilon^{+}=\varepsilon^{-}\) which we will simply denote by \(\varepsilon\).

In the cannonical setting with \(\Sigab = \rho\dul{I}\) and \(d\geq \omega(1)\), the achievable and converse regions for BHT and MAP are as follows.

\begin{center}
\begin{tabular}{ r | c || c | c}
  \(\frac{I_{\ul{X}\ul{Y}}}{\log n}\) & \multicolumn{3}{l}{\qquad\qquad 1 \,\,\,\,\qquad\qquad 2}\\
  \hline
  \multirow{2}{*}{BHT} & \(\varepsilon^{-}+\varepsilon^{+}\) & \(\varepsilon^{-} \leq o(n)\) & \(\varepsilon^{-} \leq o(n)\) \\
  & \(\geq \Omega(n)\) &  \(\varepsilon^{+} \leq o(n)\) & \(\varepsilon^{+} \leq o(1)\)\\
  \hline
  MAP & \multicolumn{2}{c||}{\(\varepsilon\geq \Omega(1)\)} & \(\varepsilon\leq o(1)\)
\end{tabular}
\end{center}

Notice the tight gap between the converse and achievability regions, corresponding to regimes \(\log n(1-o(1))\leq I_{\ul{X}\ul{Y}}\leq \log n +\calO(1)\) and \(2\log n - \calO(1) \leq I_{\ul{X}\ul{Y}} \leq 2\log n + \calO(1)\).
}


\section{MAP estimation}

\paragraph{Matching algorithm}

The maximum a posteriori estimator is the optimal estimator  for the exact matching \(M\) given \(F\). Given some realization \(\dul{f}=(\dul{a},\dul{b})\),
\begin{align*}
  \hat{m}(\dul{f})
  &= \argmax_{m}\,\,\, \Pr\croc{M=m|\dul{F}=\dul{f}}\\
  &= \argmax_{m}\,\,\, \frac{p_{\dul{F}|M}(\dul{f}|m)P_{M}(m)}{p_{\dul{F}}(\dul{f})}\\
  &\eql{a} \argmax_{m}\,\,\, p_{\dul{F}|M}(\dul{f}|m)
\end{align*}
where $(a)$ follows from the fact that $M$ has a uniform distribution.

\subsection{Achievability analysis}

We establish a sufficient condition on the mutual information \(I_{XY}\) between feature pairs  to achieve a perfect alignment. The rest of this section assumes the cannonical setting. However, by the equivalence between the general setting and the canonical setting (as shown in \hyperref[sec:featureTransformation]{Appendix \ref*{sec:featureTransformation}}), the result directly applies to the general setting.

Our analysis goes as follows: Lemma \ref{lemma:ChernoffBound} sets an upper bound on the error probability that a given matching is more likely than the actual one. This bound is in the form of a function \(R\) whose explicit value remains to be determined. Lemma \ref{lemmma:bound2determinant} gives an expression of \(R\) that has a decomposition with terms corresponding to each cycle of `mismatchings'. Finally Lemma \ref{lemma:determinant} gives the explicit expression for each of these cycle-terms and Lemma \ref{lemma:bound4determinant} bounds their product by a function whose value only depends on the number of mismatchings. Joining these results gives us the achievability condition in Theorem \ref{thm:achMAP}.

\begin{definition}
  Given any pair of bijective matchings \(m_1,m_2\subseteq\UA\times\UB\), define the event
  \begin{align*}
    \calE(m_1,m_2) = \acc{\dul{f} 
    : p_{\dul{F}|M}(\dul{f}|m_1) \leq p_{\dul{F}|M}(\dul{f}|m_2)}.
  \end{align*}
\end{definition}
Notice that given matching \(m=M\), the MAP estimator fails if and only if there exists some matching \(m'\neq m\) such that \(\dul{F}\in \calE(m,m')\).

\begin{definition}
  Given any pair of bijective matchings \(m_1,m_2\subseteq\UA\times\UB\), define the function
  \begin{align*}
    R(m_1,m_2) &\define \int \sqrt{p_{\dul{F}|M}(\dul{f}|m_1) p_{\dul{F}|M}(\dul{f}|m_2)} \mathop{d\dul{f}}
  \end{align*}
  where the integral is over the whole space $\realS^{(\UA \sqcup \UB) \times [d]}$. 
\end{definition}
\begin{lemma}
  \label{lemma:ChernoffBound}
  For any pair of bijective matchings \(m_1,m_2\subseteq\UA\times\UB\)
  \begin{align*}
    \Pr\croc{\dul{F}\in\calE(m_1,m_2)|M=m_1} \leq R(m_1,m_2)
  \end{align*}
\end{lemma}
\begin{proof}
  For any \(\theta\geq0\)
  \begin{align*}
    \Pr[\dul{F}\in\,\,\calE(m_1,m_2)|M=m_1]
                  &= \E{\ind{\frac{p_{\dul{F}|M}(\dul{f}|m_2)}{p_{\dul{F}|M}(\dul{f}|m_1)}\geq1}\Big|M=m_1}\\
                  &\leq \int \pth{\frac{p_{\dul{F}|M}(\dul{f}|m_2)}{p_{\dul{F}|M}(\dul{f}|m_1)}}^{\theta} p_{\dul{F}|M}(\dul{f}|m_1) \mathop{d\dul{f}}\\
                  &= \int (p_{\dul{F}|M}(\dul{f}|m_2))^{\theta} (p_{\dul{F}|M}(\dul{f}|m_1))^{1-\theta}\mathop{d\dul{f}}
  \end{align*}
  Selecting \(\theta=1/2\) gives the claim.
\end{proof}

\begin{definition}
  \label{def:shiftedIdentity}
  Define \textit{shifted identity matrices} \(\dul{I}^{(k,+)}\) and \(\dul{I}^{(k,-)}\) of size \(k\) as
  \begin{align*}
    \dul{I}^{(k,+)}_{i,j} &= \ind{j-i=1 \mod k}\\
    \dul{I}^{(k,-)}_{i,j} &= \ind{j-i=-1 \mod k}.
  \end{align*}
  We simply write \(\dul{I}^{(+)}\) and \(\dul{I}^{(-)}\) when there is no need to specify the size of the matrix.

  \label{def:canonicalLaplacian}
  For any \(\ell\in\natS^+\), 
  \begin{align*}
    \dul{L}^\ell(s,t) \define s\dul{I}^\ell - \frac{t}{2}\pth{\dul{I}^{(\ell,+)}+\dul{I}^{(\ell,-)}},
  \end{align*}
  where \(s,t\in\realS\).
\end{definition}

\begin{lemma}
  \label{lemmma:bound2determinant}
  Suppose $\da = \db =1$ and $\dul{\Sigma} = \crocMat{1}{\rho}{\rho}{1}$.
  For bijective matchings \(m_1,m_2\subseteq \UA\times\UB\), 
  \begin{align*}
    R(m_1,m_2) = \pth{1-\rho^2}^{\frac{n}{2}}\prod_{\ell} \croc{\det \dul{L}^\ell\pth{1-\frac{\rho^2}{2},\frac{\rho^2}{2}}}^{-\frac{k_\ell}{2}}
  \end{align*}
  where \(k_\ell\) is the number of cycles of length \(\ell\) of permutation \(m_1 \circ m_2^\top \subseteq \UA \times \UA\).
\end{lemma}
\vspace{-0.5cm}
\begin{proof}
  For a matching \(m \subseteq \UA\times\UB\), let \(\dul{m}\in\{0,1\}^{\UA\times\UB}\) be the indicator matrix for \(m\).

  Because $\da = \db = 1$, we will treat the databases as vectors $\ul{A} \in \realS^{\UA}$ and $\ul{B} \in \realS^{\UB}$.
  Let $\ul{F} \in \realS^{\UA \sqcup \UB}$ be the concatenation of $\ul{A}$ and $\ul{B}$.
  Observe that $\dul{\Sigma}^{-1} = \frac{1}{1-\rho^2}\crocMat{1}{-\rho}{-\rho}{1}$.
  Then we can write
  \begin{align}
    p_{\ul{F}|M}((\ul{a},\ul{b})|m) =
    \frac{1}{\bpth{2\pi\sqrt{1-\rho^2}}^n}\cdot
    \exp\pth{-\frac{1}{2(1-\rho^2)}\crocVec{\ul{a}}{\ul{b}}^\top \crocMat{\dul{I}^{\calU}}{-\rho\dul{m}}{-\rho\dul{m}^\top}{\dul{I^{\calV}}} \crocVec{\ul{a}}{\ul{b}}}.
    \label{f-density}
  \end{align}
  For compactness, call the matrix that appears in \eqref{f-density} \(\dul{\Sigma}(m)\). 
  This gives us
  \begin{align*}
    &\eq\pth{p_{\ul{F}|M}\bpth{\ul{f};m_1}p_{\ul{F}|M}\bpth{\ul{f};m_2}}^{\frac{1}{2}}
    =\frac{1}{\bpth{2\pi\sqrt{1-\rho^2}}^n}\cdot
    \exp\pth{-\frac{\ul{f}^\top \croc{\dul{\Sigma}(m_1)+\dul{\Sigma}(m_2)} \ul{f}}{4(1-\rho^2)}}.
  \end{align*}
  We obtain $R(m_1,m_2)$ by integrating this expression over the whole space:
  \begin{align}
    R(m_1,m_2)
    &= \int \sqrt{p_{\ul{F}|M}\bpth{\ul{f};m_1}p_{\ul{F}|M}\bpth{\ul{f};m_2}}df \nonumber\\
    &=\croc{\frac{\pth{1-\rho^2}^n}{\det\pth{\frac{1}{2}\dul{\Sigma}(m_1)+\frac{1}{2}\dul{\Sigma}(m_2)}}}^{1/2}. \label{eqn:achMap1}
  \end{align}
  Observe that \(\crocMat{\dul{I}}{\dul{z}}{\dul{z}^\top}{\dul{I}} = \crocMat{\dul{I}}{\dul{0}}{\dul{z}^\top}{\dul{I}}\crocMat{\dul{I}}{\dul{z}}{\dul{0}}{\dul{I}-\dul{z}^\top\dul{z}}\) for any matrix \(\dul{z}\). Then \(\det\crocMat{\dul{I}}{\dul{z}}{\dul{z}^\top}{\dul{I}} = \det\bpth{\dul{I}-\dul{z}^\top\dul{z}}\). Using this relation we have
  \begin{align}
    \det\pth{\frac{1}{2}\dul{\Sigma}(m_1)+\frac{1}{2}\dul{\Sigma}(m_2)}
    &= \det\crocMat{\dul{I}}{-\frac{\rho}{2}\bpth{\dul{m}_1+\dul{m}_2}}{-\frac{\rho}{2}\bpth{\dul{m}_1+\dul{m}_2}^\top}{\dul{I}} \nonumber\\
    &= \det\pth{\dul{I}-\frac{\rho^2}{4}\bpth{\dul{m}_1+\dul{m}_2}^\top\bpth{\dul{m}_1+\dul{m}_2}} \nonumber\\
    &= \det\pth{\pth{1-\frac{\rho^2}{2}}\dul{I}-\frac{\rho^2}{4}\pth{\dul{m}_1^\top\dul{m}_2+\dul{m}_2^\top\dul{m}_1}}. \label{eqn:achMap2}
  \end{align}
  
  Notice that \(\dul{m}_1^\top\dul{m}_2\in\{0,1\}^{\UA\times\UA}\) is the permutation matrix corresponding to permutation \(\pi = m_1\circ m_2^\top\) described in the statement of the lemma. Let \(\calC\) be the set of cycles of \(\pi\) and \(\{\ell_c\}_{c\in\calC}\) denote their lengths. Consider the cycle notation of this permutation, i.e. \((u_1,u_2,\cdots,u_{\ell_c})(u_1',\cdots,u_{\ell_{c'}}')\cdots\), and specify an ordering of \(\UA\) based on this expression: \(u_1,u_2,\cdots,u_{\ell_c},u_1',\cdots,u_{\ell_{c'}}',\cdots\). Given this ordering of rows and columns, the permutation matrix \(\dul{m}_1^\top\dul{m}_2\) has block diagonal matrix form, with one block for each cycle \(c\in\calC\) and every block having the form of a shifted identity matrix \(\dul{I}^{(\ell_c,+)}\). Then \(\dul{m}_2^\top\dul{m}_1 = \bpth{\dul{m}_1^\top\dul{m}_2}^\top\) has the same block diagonal form with the shifted identity matrices \(\dul{I}^{(\ell_c,-)}\), since \(\dul{I}^{(\ell_c,-)} = \bpth{\dul{I}^{(\ell_c,+)}}^\top\).
  
  The determinant of a block diagonal matrix is equal to the product of the determinants of each block. Then we have
  \begin{align*}
    \det\pth{\pth{1-\frac{\rho^2}{2}}\dul{I}-\frac{\rho^2}{4}\pth{\dul{m}_1^\top\dul{m}_2+\dul{m}_2^\top\dul{m}_1}}
    &= \prod_{c\in\calC} \det\pth{\pth{1-\frac{\rho^2}{2}}\dul{I} - \frac{\rho^2}{4}\pth{\dul{I}^{(\ell_c,+)}+\dul{I}^{(\ell_c,-)}}}\\
    &= \prod_{c\in\calC} \det\pth{\dul{L}^{\ell_c}\pth{1-\frac{\rho^2}{2},\frac{\rho^2}{2}}}\\
    &= \prod_{\ell\in[n]} \croc{\det\pth{\dul{L}^{\ell}\pth{1-\frac{\rho^2}{2},\frac{\rho^2}{2}}}}^{k_\ell},
  \end{align*}
  where \(k_\ell\) denotes the number of cycles of length \(\ell\) in the permutation \(\pi\). Combining this with (\ref{eqn:achMap1}) and (\ref{eqn:achMap2}) gives us the claimed result.
\end{proof}

\begin{lemma}
  \label{lemma:determinant}
  For any \(\ell\in\natS^+\),
  \begin{align*}
    \det\pth{\dul{L}^\ell(s,t)} &= \prod_{j\in[\ell]} \croc{s-t\cdot\cos\pth{j\frac{2\pi}{\ell}}}.
  \end{align*}
  In particular
  \begin{align*}
    \det\pth{\dul{L}^{1}(s,t)} = s-t \qquad \textrm{ and } \qquad \det\pth{\dul{L}^{2}(s,t)} = s^2-t^2
  \end{align*}
\end{lemma}
\begin{proof}
  Let \(\ul{z}^{\,k}\in\compS^\ell\) denote a family of vectors such that for any \(k\in [\ell]\), \(\ul{z}^{\,k}_j = e^{2\pi i\frac{jk}{\ell}}\), where \(i^2=-1\). Observe that
  \begin{align*}
    \dul{I}^{(+)}\ul{z}^{\,k} &= e^{2\pi i \frac{k}{\ell}} \ul{z}^{\,k} \qquad\textrm{ and }\qquad
    \dul{I}^{(-)}\ul{z}^{\,k} = e^{-2\pi i \frac{k}{\ell}} \ul{z}^{\,k}
  \end{align*}
  Vectors \(\ul{z}^{\,k}\) are the eigenvectors of \(\dul{L}^k(s,t)\):
  \begin{align*}
    \dul{L}^\ell(s,t)\,\,\ul{z}^{\,k}
    &= \croc{s\cdot \dul{I} - \frac{t}{2}\pth{\dul{I}^{(+)}+\dul{I}^{(-)}}}\ul{z}^{\,k}\\
    &= \croc{s -\frac{t}{2}\pth{e^{2\pi i \frac{k}{\ell}}+e^{-2\pi i \frac{k}{\ell}}}}\ul{z}^{\,k}\\
    &= \croc{s-t\cdot\cos\pth{2\pi\frac{k}{\ell}}}\ul{z}^{\,k}
  \end{align*}
  We compute the determinant by taking the product of the \(\ell\) eigenvalues (one for each \(k\in[\ell]\)).
\end{proof}

\begin{lemma}
  \label{lemma:bound4determinant}
  For any \(\ell\in\natS\setminus\{0,1\}\) and \(s,t\in\realS\) such that \(s>|t|\),
  \begin{align*}
    \det\croc{\dul{L}^\ell(s,t)} \geq \pth{\det\croc{\dul{L}^2(s,t)}}^{\ell/2}
  \end{align*}
  \bcomment{where \(\dul{L}^\ell\) and \(\dul{L}^2\) refer to the matrix valued functions as defined in Definition \ref{def:canonicalLaplacian}}
\end{lemma}
\begin{proof}
  First note that, by Lemma \ref{lemma:determinant},
  \begin{align*}
    \det\croc{\dul{L}^2(s,t)} = s^2-t^2.
  \end{align*}

  We want to bound the determinant of the matrix \(\dul{L}^\ell(s,t)\), which is equal to the product of its eigenvalues \(\pth{\lambda_j}_{j\in[\ell]}\). The sum of eigenvalues is equal to the trace of the matrix, which is known, since all diagonal elements of \(\dul{L}^{\ell}(s,t)\) equal \(s\) for any \(\ell\geq 2\). So \(\sum \lambda_k = \tr\bpth{\dul{L}^\ell(s,t)} = s\ell\). Furthermore, observe that all eigenvalues are in the range \([s-t,s+t]\). Consider a sequence formed of two copies of each eigenvalue \(\lambda_i\). This sequence has mean \(s\) and has all entries within the range \([s-t,s+t]\). Then, as it is proven in \hyperref[lemma:elementary]{Lemma \ref*{lemma:elementary}},
  \begin{align*}
    \prod_{j\in[2\ell]} \lambda_j^2 \geq \pth{s-t}^\ell \pth{s+t}^\ell
  \end{align*}
  Taking the square root of both sides results in the claim.
\end{proof}

\newtheorem*{tAchMAP}{Theorem \ref*{thm:achMAP}}
\begin{tAchMAP}
  \thmAchMAP
\end{tAchMAP}
\begin{proof}
  Recall the canonical setting where \(\Siga=\Sigb=I\) and \(\Sigab=\diag(\ul{\rho})\).
  
  Let $R_i(m,m')$ be the value of $R(m,m')$ when $\dul{\Sigma} = \crocMat{1}{\ul{\rho}_i}{\ul{\rho}_i}{1}$.
  By the union bound
  \begin{align*}
    \eq\Pr\Big[F\in\bigcup_{m'\neq m}\calE(m,m')|M=m\Big]
    &\leq \sum_{m'\neq m}\Pr\croc{F\in\calE(m,m')|M=m}\\
    & \stackrel{(a)}{\leq} \sum_{m'\neq m} R(m,m') = \sum_{m'\neq m}\prod_{i\in[d]} R_i(m,m')\\
    & \stackrel{(b)}{=} \sum_{m'\neq m}\prod_{i\in[d]} \croc{\frac{\pth{1-\rho_i^2}^n}{\prod_\ell \croc{\det\pth{\dul{L}^\ell\pth{1-\frac{\rho_i^2}{2},\frac{\rho_i^2}{4}}}}^{k_\ell}}}^{\frac{1}{2}}\\
    &\stackrel{(c)}{\leq} \sum_{m'\neq m}\prod_{i\in[d]}\croc{\frac{\pth{1-\rho_i^2}^n}{\pth{s-t}^{|m\cap m'|}\pth{s^2-t^2}^{\frac{1}{2}\pth{n-|m\cap m'|}}}}^{\frac{1}{2}}\\
    &= \sum_{m'\neq m}\prod_{i\in[d]} \pth{1-\rho_i^2}^{\frac{n-|m\cap m'|}{4}},
  \end{align*}
  where (a) follows from Lemma \ref{lemma:ChernoffBound}, (b) follows from Lemma \ref{lemmma:bound2determinant}, with \(k_\ell\) denoting the number of cycles of length \(\ell\) in the permutation \(m'\circ m^\top\), and (c) follows from Lemmas \ref{lemma:determinant} and \ref{lemma:bound4determinant}, with \(s=1-\frac{\rho_i^2}{2}\) and \(t=\frac{\rho_i^2}{2}\), which gives us \(s-t = s^2-t^2 = 1-\rho_i^2\).
  
  Given any \(k\in\natS\) there are exactly \((!k)\times\binom{n}{k}\) different matchings \(m'\) such that \(k=n-|m\cap m'|\), where \((!k)\) represents the number of derangements over a set of size \(k\). We bound \((!k)\times\binom{n}{k}\leq n^k\). Thus
  \begin{align*}
    \Pr\Big[F\in\bigcup_{m'\neq m}\calE(m,m')|M=m\Big]
    \leq \sum_{k\in\natS} \,n^k \cdot \prod_{i\in[d]} \pth{1-\rho_i^2}^{k/4}
  \end{align*}
  If \(n \prod_{i \in [d]} \pth{1-\rho_i^2}^{\frac{1}{4}} \leq o(1)\), then by summing the geometric series, we see that the above expression is \(o(1)\). Therefore
  \begin{align*}
    \exp\pth{-I_{XY}} = \prod_{i\in[d]}\pth{1-\rho_i^2}^{\frac{1}{2}} \leq o(1/n^2)
  \end{align*}
  is a sufficient condition for exact recovery under the canonical setting. Taking the logarithm of both sides gives us the claimed result.
\end{proof}

\subsection{Converse analysis}
We establish a necessary condition on the mutual information \(I_{XY}\) between feature pairs  to achieve a perfect alignment.
\begin{lemma}
  \label{lemma:Cramer}
  Let \(d\in\natS\) such that \(d=\omega(1)\) as well as \(\dul{\Sigma}_A=\dul{\Sigma}_B=\dul{I}^d\) and \(\dul{\Sigma}_{AB}=\rho\dul{I}\). Given bijective matchings \(m_1,m_2\subset \UA\times\UB\) such that \(|m_1\cap m_2| = n-2\),
  \begin{align*}
    \Pr[F\in\calE(m_1,m_2)|M=m_1] \geq (1-\rho_i^2)^{\frac{d}{2}(1+o(1))}.
  \end{align*}
\end{lemma}
\begin{proof}
  Consider the conditional generating function
  \begin{align*}
    c_i(\theta)
    &= \E{\exp\pth{\theta\log\frac{p_{\dul{F}_{*i}|M}(\dul{f}_i|m_2)}{p_{\dul{F}_{*i}|M}(\dul{f}_i|m_1)}}\Big| M=m_1}\\
    &= \int \pth{\frac{p_{\dul{F}_{*i}|M}(\dul{f}_i|m_2)}{p_{\dul{F}_{*i}|M}(\dul{f}_i|m_1)}}^{\theta} p_{\dul{F}_{*i}|M}(\dul{f}_i|m_1) d\dul{f}_i\\
    &= \int (p_{\dul{F}_{*i}|M}(\dul{f}_i|m_2))^{\theta} (p_{\dul{F}_{*i}|M}(\dul{f}_i|m_1))^{1-\theta} d\dul{f}_i
  \end{align*}
  The generating function is minimized at \(\theta = 1/2\) in which case we get \(c_i(\theta) = R_i(m_1,m_2)\).
  
  We evaluate the value of this function using Lemmas \ref{lemmma:bound2determinant} and \ref{lemma:determinant} with \(s=1-\rho^2/2\) and \(t=\rho^2/2\). By \(|m_1\cap m_2| = n-2\) we get \(R_i(m_1,m_2) = \sqrt{1-\rho^2}\).
  
  By Cramér's Theorem on the asymptotic tightness of the Chernoff bound (see for example \cite{hajek}), there is some \(\epsilon(d) \leq o(1)\) such that
  \begin{align*}
    \eq \Pr\croc{\log\frac{p_{\dul{F}|M}(\dul{f}|m_2)}{p_{\dul{F}|M}(\dul{f}|m_1)}\geq 0}
    &= \Pr\croc{\sum_{i \in [d]} \log\frac{p_{\dul{F_{*i}}|M}(\dul{f}_{*i}|m_2)}{p_{\dul{F}_{*i}|M}(\dul{f}_{*i}|m_1)}\geq 0}\\
    &\geq \exp\pth{-d\croc{\epsilon - \inf_\theta \log c_i(\theta)}}\\
    &= \exp\pth{d (\log R_i(m_1,m_2) - \epsilon}\\
    &\geq \pth{R_i(m_1,m_2)}^{d(1-o(1))}\\
    &= \pth{1-\rho^2}^{\frac{d}{2}(1-o(1))}
  \end{align*}
\end{proof}

\begin{lemma}
  \label{lemma:errorIntersection}
  If \(\dul{\Sigma}_A=\dul{\Sigma}_B=I\) and \(\dul{\Sigma}_{AB}=\rho\dul{I}\), then given any bijective matchings \(m_1,m_2,m_3\in\UA\times\UB\)
  \begin{align*}
    \Pr[F\in\calE(m_1,m_2)\cap \calE&(m_1,m_3)|M=m_1]
                                    \leq \pth{1-\rho_i^2}^{\frac{d}{4}\pth{n-|m_2\cap m_3|}}.
  \end{align*}
\end{lemma}
\begin{proof}
  We will abbreviate $p_{\dul{F}|M}(\cdot|\cdot)$ as $p(\cdot|\cdot)$.
  
  For any \(\theta,\theta'>0\) we have
  \begin{align*}
    \Pr[F\in\calE(m_1,m_2)\cap \calE(m_1,m_3)|M=m_1]
    &= \E{\ind{\frac{p(\dul{f}|m_2)}{p(\dul{f}|m_1)}\geq1,\frac{p(\dul{f}|m_3)}{p(\dul{f}|m_1)}\geq1}\Big|M=m_1}\\
    &\leq \int \pth{\frac{p(\dul{f}|m_2)}{p(\dul{f}|m_1)}}^{\theta}\pth{\frac{p(\dul{f}|m_3)}{p(\dul{f}|m_1)}}^{\theta'} p(\dul{f}|m_1) d\dul{f}\\
    &= \int (p(\dul{f}|m_2))^{\theta} (p(\dul{f}|m_3))^{\theta'} (p(\dul{f}|m_1))^{1-\theta-\theta'} d\dul{f}.
  \end{align*}
  The choice of \(\theta=\theta'=1/2\) gives the upper bound as \(R(m_2,m_3)\). We evaluate this function using Lemmas \ref{lemmma:bound2determinant} and \ref{lemma:determinant}, which give us the claimed result.
\end{proof}

\newtheorem*{tConMAP}{Theorem \ref*{thm:conMAP}}
\begin{tConMAP}
  \thmConMAP
\end{tConMAP}
\begin{proof}
  Let \(\calM^\calE(f,m) \define \acc{m'|f\in\calE(m,m'), m'\neq m}\) denote the set of matches that are at least as likely as \(m\) under the database instance \(f\). The MAP algorithm succeeds if and only if \(\calM^\calE(F,M)=\emptyset\).
  
  Also define \(\calM_2(m) \define \acc{m'\big||m\cap m'| = n-2}\).
  For compactness, let \(X \define |\calM^\calE(f,m) \cap \calM_2(m)|\).
  Clearly \(0 \leq X \leq |\calM^\calE(f,m)|\).
  
  We apply Chebyshev's inequality:
  \begin{align*}
    \Pr[|\calM^\calE(F,M)| = 0] \leq \Pr[X = 0]
    \leq \Pr\croc{\pth{X-\mathbb{E}X}^2 \geq \mathbb{E}^2X}
    \leq \frac{\Var X}{\mathbb{E}^2 X}
  \end{align*}
  
  All matchings are equally likely. Therefore, given any bijective matching \(m\in\UA\times\UB\),
  \begin{align*}
    \mathbb{E}\norm{\calM_2^\calE(F,M)} =\sum_{m'\in \calM_2(m)} \Pr[F\in\calE(m,m')|M=m]
  \end{align*}
  Let \(\varepsilon_1 \define \Pr[F\in\calE(m,m')|M=m]\) given \(|m\cap m'| = n-2\). Notice that this probability does not depend on the choice of \(m'\in\calM_2(m)\). Then \(\mathbb{E}\norm{\calM_2^\calE(F,M)} = |\calM_2(m)|\cdot \varepsilon_1 = \binom{n}{2}\cdot\varepsilon_1\).
  \begin{align*}
    |\calM_2^\calE(f,m)|^2
    &= \croc{\sum_{m'\in \calM_2(m)}\ind{f\in \calE(m,m')}}^2\\
    &= \sum_{m'\in \calM_2(m)} \ind{\calE(m,m')}
     + 2\times \sum_{\{m',m''\} \subset \calM_2(m)} \ind{\calE(m,m'),\calE(m,m'')}
  \end{align*}
  There are \(3\binom{n}{4}\) different ways to choose to matchings \(\{m',m''\}\subset \calM_2(m)\) such that \(|m'\cap m''| = n-4\), and \(3\binom{n}{3}\) ways to choose them such that \(|m'\cap m''| = n-3\).
  Notice that \(3\binom{n}{4}+3\binom{n}{3} = \binom{|\calM_2(m)|}{2}\) and these partition all the choices for \(\{m',m''\}\subset \calM_2(m)\).

  When \(|m'\cap m''| = n-4\), the error events become independent and we get \[\Pr[F\in\calE(m,m')\cap \calE(m,m'')|M=m] = \varepsilon_1^2.\]
  Let \(\varepsilon_2 \define \Pr[F\in\calE(m,m')\cap \calE(m,m'')|M=m]\) given \(|m'\cap m''| = n-3\).
  
  By the relation \(z + 2\binom{z}{2} = z^2\). For \(z = |\calM_2(m)| = \binom{n}{2}\) and \(\binom{z}{2} = 3\binom{n}{3} + 3\binom{n}{4}\) we can write:
  \begin{align*}
    \mathbb{E}^2\norm{\calM_2^\calE(F,M)}
    &= |\calM_2(m)|^2\varepsilon_1^2\\
    &= \binom{n}{2}\varepsilon_1^2 + \croc{6\binom{n}{3}+6\binom{n}{4}}\varepsilon_1^2\\
    \E{|\calM_2^\calE(F,M)|^2} &= \binom{n}{2} \varepsilon_1 + 6\binom{n}{3}\varepsilon_2 + 6\binom{n}{4}\varepsilon_1^2\\
    \Var\norm{\calM_2^\calE(F,M)} &= \binom{n}{2}(\varepsilon_1-\varepsilon_1^2) + 6\binom{n}{3}(\varepsilon_2-\varepsilon_1^2)\\
    &\leq \binom{n}{2}\varepsilon_1 + 6\binom{n}{3}\varepsilon_2
  \end{align*}
  Plugging these values into the Chernoff bound we get
  \begin{align*}
    \Pr[|\calM^\calE(F,M)| = 0]
    &\leq \frac{\binom{n}{2}\varepsilon_1 + 6\binom{n}{3}\varepsilon_2}{\binom{n}{2}^2\varepsilon_1^2}
    \leq \calO\pth{\frac{1}{n^2\varepsilon_1}+\frac{\varepsilon_2}{n\varepsilon_1^2}}
  \end{align*}
  By lemma \ref{lemma:Cramer} and \ref{lemma:errorIntersection} we have \(\varepsilon_1\geq (1-\rho_i^2)^{\frac{d}{2}(1-o(1))}\) and \(\varepsilon_2 \leq (1-\rho_i^2)^{3d/4}\).
  Thus \(\varepsilon_1^2/\varepsilon_2 \geq (1-\rho^2)^{\frac{d}{4}(1-o(1))}\).

  If $(1-\rho_i^2)^{\frac{d}{2}} \geq n^{-2 + \Omega(1)}$, then
  \begin{align*}
      n\varepsilon_1^2/\varepsilon_2 \geq n(1-\rho^2)^{\frac{d}{4}(1-o(1))} &\geq n^{1+\frac{1}{2}(1-o(1))(-2+\Omega(1))} \geq n^{\Omega(1)}\\
      \textrm{and } n^2\epsilon_1 = n^2 (1-\rho_i^2)^{\frac{d}{2}(1-o(1))} &\geq n^{2 + (1-o(1))(-2+\Omega(1))} \geq n^{\Omega(1)}
  \end{align*}
  and therefore $\Pr[|\calM^\calE(F,M)| = 0] \leq \mathcal{O}(n^{-\Omega(1)}) \leq o(1)$.
\bcomment{
  Let \(\eta\leq \calO(1)\) such that \((1-\rho_i^2)^d \geq n^{-2+\eta}\).
  Then \(\varepsilon_1 \geq n^{-2+\eta - o(1)}\) and \(\varepsilon_2/\varepsilon_1 \leq n^{-1+
    \eta/2+o(1)}\). This gives us
  \begin{align*}
    \Pr[e(F,M) = 0]
    &\leq \calO\pth{\frac{n^{-2+\eta/2+o(1)}}{\varepsilon_1}}\\
    & \leq \calO\pth{n^{-\eta/2 +o(1)}}
  \end{align*}
  If \(\eta \geq \omega(1/\log n)\) then \(\Pr[e(F,M) = 0] \leq o(1)\).
  This corresponds to \(I_{XY} \leq (2-\eta)\log n \leq 2\log n - \omega(1)\).

  From Lemma~\ref{lemma:three-matchings} we have $\epsilon_2 \leq (b^{\circ}_2(z,z))^{\frac{3}{2}}$ and from Lemma~\ref{lemma:event-lb} we have $\epsilon_1 \geq (b^{\circ}_2(z,z))^{1 + o(1)}$, so
  \begin{equation*}
    \Pr[ = 0] \leq \mathcal{O}\left(\frac{1}{n^2 (b^{\circ}_2(z,z))^{1 + o(1)}} + \frac{1}{n(b^{\circ}_2(z,z))^{\frac{1}{2} + o(1)}}\right).
  \end{equation*}
  If $b^{\circ}_2(z,z) \geq n^{-2 + \Omega(1)}$, then
  \[
    n^2 b^{\circ}_2(z,z)^{1+o(1)} \geq n^{2 + (1+o(1))(-2+\Omega(1))} \geq n^{\Omega(1)} \geq \omega(1)
  \]
  and $\Pr[X = 0] \leq o(1)$.
}
\end{proof}

\section{Binary hypothesis testing}

\paragraph{Matching algorithm}

We consider an algorithm that does gives us a `matching' \(\hat{m}\subseteq\UA\times\UB\) that is not necessarily bijective, i.e. any entry can have multiple matches in the other dataset.

Recall that we denote the \(j\)-th row of a matrix \(\dul{z}\) by \(\ul{z}_{j*}\).

Given some \(\dul{a}\in\realS^{\UA\times[d_1]}\) and \(\dul{b}\in\realS^{\UB\times[d_2]}\) and \(f=\bpth{\dul{a},\dul{b}}\) the estimated `matching' is given by
\begin{align*}
  \hat{m}(f)= \acc{(u,v)\in\UA\times\UB\Big|(\ul{a}_{u*}^\top,\ul{b}_{v*}^\top)\in H_\tau }.
\end{align*}
\(H_\tau\) is the log ratio test given by
\begin{align*}
  H_\tau = \acc{(\ul{x},\ul{y})\in\realS^d\times\realS^d\Big|\log\frac{p_{\ul{X}\ul{Y}}(\ul{x},\ul{y})}{p_{\ul{X}}(\ul{x})p_{\ul{Y}}(\ul{y})}\geq \tau}
\end{align*}
where \(p_{\ul{X}}\) and \(p_{\ul{Y}}\) denote the probability density functions of feature vectors associated with identifiers in \(\UA\) and \(\UB\) respectively, and \(\tau\in\realS\) is some constant to be determined.

\subsection{Achievability analysis}

\bcomment{
  We assume the special problem setting with \(\Siga=\Sigb=\dul{I}^d\) and \(\Sigab=\diag(\ul{\rho})\) where \(-1<\rho_i<1\) for all \(i\in[d]\) and establish sufficient conditions on the mutual information \(I_{\ul{X}\ul{Y}}\) between feature pairs. Under these assumptions the probability density functions of feature vectors from both sides are identical, therefore we denote both simply by \(q(\ul{z}) = p_{\ul{X}}(\ul{z})=p_{\ul{Y}}(\ul{z})\).
}

In our analysis we establish upper and lower bounds on the threshold \(\tau\) that allow given probability bounds on false negatives and false positives.
The mean and variance of the log ratio random variable were computed in Section~\ref{section:corr}.
Using these values we get an upper bound on the probability of false negatives in Lemma \ref{lemma:FN} by the Chebyshev inequality. Lemma \ref{lemma:FP} gives an upper bound on the number of false positives. Finally, taking the intersection of the conditions on \(\tau\) allows us to derive the achievability result given in Theorem \ref{thm:achLRT}.

\begin{lemma}
  \label{lemma:FN}
  If \(\tau \leq I_{\ul{X}\ul{Y}} - \sigma_{\ul{X}\ul{Y}}/\sqrt{\varepsilon}\) then
  \begin{align*}
    \Pr\croc{(\dul{A}_{u*}^\top,\dul{B}_{v*}^\top)\notin H_\tau|(u,v)\in M} \leq \varepsilon.
  \end{align*}
\end{lemma}
\begin{proof}
  Let \((u,v)\in M\) and \((\ul{X},\ul{Y})=(\dul{A}_{u*}^\top,\dul{B}_{v*}^\top)\). Given \(\mu = I_{\ul{X}\ul{Y}} = \E{\log \frac{p_{\ul{X}\ul{Y}}(\ul{X},\ul{Y})}{p_{\ul{X}}(\ul{X})p_{\ul{Y}}(\ul{Y})}} = I_{\ul{X}\ul{Y}}\) and \(\sigma^2 = \sigma_{\ul{X}\ul{Y}}^2 = \Var\pth{\log \frac{p_{\ul{X}\ul{Y}}(\ul{X},\ul{Y})}{p_{\ul{X}}(\ul{X})p_{\ul{Y}}(\ul{Y})}}\), by Chebyshev's inequality we get
  \begin{align*}
    \Pr\croc{\norm{\mu-\log \frac{p_{\ul{X}\ul{Y}}(\ul{X},\ul{Y})}{p_{\ul{X}}(\ul{X})p_{\ul{Y}}(\ul{Y})}}\geq \frac{\sigma}{\sqrt{\varepsilon}}} \leq \varepsilon
  \end{align*}
  This probability is lower bounded by \(\Pr\left[\mu-\log \frac{p_{\ul{X}\ul{Y}}(\ul{X},\ul{Y})}{p_{\ul{X}}(\ul{X})p_{\ul{Y}}(\ul{Y})}\geq \frac{\sigma}{\sqrt{\varepsilon}}\right]\) which is equal to the probability \(\Pr\left[(\dul{A}_{u*}^\top,\dul{B}_{v*}^\top)\notin H_\tau|(u,v)\in M\right]\) for \(\tau = \mu-\sigma/\sqrt{\varepsilon}\). Then this choice of \(\tau\), or any smaller value, is a sufficient condition to bound the error probability by \(\varepsilon\).
\end{proof}

\begin{lemma}
  \label{lemma:FP}
  Given any \(\tau\in\realS\),
  \begin{align*}
    \Pr\croc{(\dul{A}_{u*}^\top,\dul{B}_{v*}^\top)\in H_\tau|(u,v)\notin M} \leq e^{-\tau}
  \end{align*}
\end{lemma}
\begin{proof}
  Let \((u,v)\in M\) and \((\ul{X},\ul{Y})=(\dul{A}_{u*}^\top,\dul{B}_{v*}^\top)\). By Markov's inequality we get
  \begin{align*}
    \Pr\croc{\log \frac{p_{\ul{X}\ul{Y}}(\ul{X},\ul{Y})}{p_{\ul{X}}(\ul{X})p_{\ul{Y}}(\ul{Y})}\geq \tau} 
                                                                                                          &\leq e^{-\tau}\cdot\E{\frac{p_{\ul{X}\ul{Y}}(\ul{X},\ul{Y})}{p_{\ul{X}}(\ul{X})p_{\ul{Y}}(\ul{Y})}}
  \end{align*}
  We calculate the mean:
  \begin{align*}
    \E{\frac{p_{\ul{X}\ul{Y}}(\ul{X},\ul{Y})}{p_{\ul{X}}(\ul{X})p_{\ul{Y}}(\ul{Y})}}
    &= \int_{\ul{X},\ul{Y}} p_{\ul{X}}(\ul{X})p_{\ul{Y}}(\ul{Y})\cdot \frac{p_{\ul{X}\ul{Y}}(\ul{X},\ul{Y})}{p_{\ul{X}}(\ul{X})p_{\ul{Y}}(\ul{Y})} d(\ul{X},\ul{Y})\\
    &= \int_{\ul{X},\ul{Y}} p_{\ul{X}\ul{Y}}(\ul{X},\ul{Y}) d(\ul{X},\ul{Y})
  \end{align*}
  which equals to 1 since \(p_{\ul{X}\ul{Y}}(\ul{X},\ul{Y})\) is as a probability density function.
\end{proof}

\paragraph{Proof of Theorem \ref{thm:achLRT}}

By Lemma \ref{lemma:FN}, if \(\tau \leq I_{\ul{X}\ul{Y}} - \sigma_{\ul{X}\ul{Y}}/\sqrt{\varepsilon}\), then the probability that any correct match is not included in \(H_\tau\) is upper bounded by \(\varepsilon_{FN}/n\). There are \(n\) correct matches in \(\UA\times\UB\). Then the expected number of correct matches not included in \(H_\tau\), i.e. the expected number of false negatives, is upper bounded by \(\varepsilon_{FN}\).

By Lemma \ref{lemma:FP}, if \(\tau \geq \log\pth{n^2/\varepsilon_{FP}}\), then the probability that any incorrect match is included in \(H_\tau\) is upper bounded by \(\varepsilon_{FP}/n^2\). There are \(\binom{n}{2}<n^2\) incorrect matches in \(\UA\times\UB\). Then the expected number of incorrect matches included in \(H_\tau\), i.e. the expected number of false positives,  is upper bounded by \(\varepsilon_{FP}\).

A choice for \(\tau\in\realS\) satisfying both conditions exists if and only if the condition in the theorem statement holds. \hfill \qedsymbol

\bcomment{
  \newtheorem*{tAchLRT}{Theorem \ref*{thm:achLRT}}
  \begin{tAchLRT}
    \thmAchLRT
  \end{tAchLRT}

  \begin{proof}
    By Lemma \ref{lemma:FN}, if \(\tau \leq I_{\ul{X}\ul{Y}} - \sigma/\sqrt{\varepsilon}\), then the probability that any correct match is not included in \(H_\tau\) is upper bounded by \(\varepsilon_{FN}/n\). There are \(n\) correct matches in \(\UA\times\UB\). Then the expected number of correct matches not included in \(H_\tau\), i.e. the expected number of false negatives, is upper bounded by \(\varepsilon_{FN}\).
    
    By Lemma \ref{lemma:FP}, if \(\tau \geq \log\pth{n^2/\varepsilon_{FP}}\), then the probability that any incorrect match is included in \(H_\tau\) is upper bounded by \(\varepsilon_{FP}/n^2\). There are \(\binom{n}{2}<n^2\) incorrect matches in \(\UA\times\UB\). Then the expected number of incorrect matches included in \(H_\tau\), i.e. the expected number of false positives,  is upper bounded by \(\varepsilon_{FP}\).
    
    A choice for \(\tau\in\realS\) satisfying both conditions exists if and only if the condition in the theorem statement holds.
  \end{proof}
}

\subsection{Converse analysis}

We present a converse on the performance of the binary hypothesis testing algorithm based on Fano's inequality.

\begin{lemma}
    \label{lemma:conditionalEntropy}
  For $u \in \UA$ and $v \in \UB$, $H(\dul{M}_{u,v}|\dul{A}_u,\dul{B}_v) \geq \frac{\log n - I_{\ul{X}\ul{Y}}}{n}$.
\end{lemma}
\begin{proof}
  We have
  \begin{align*}
    H(\dul{M}_{u,v}|\dul{A}_u,\dul{B}_v) = H(\dul{M}_{u,v}) + I(\dul{A}_u;\dul{B}_v)
    - I(\dul{M}_{u,v};\dul{A}_u) - I(\dul{M}_{u,v};\dul{B}_v) - I(\dul{A}_u;\dul{B}_v|\dul{M}_{u,v}).
  \end{align*}
  Then $I(\dul{M}_{u,v};\dul{A}_u) = I(\dul{M}_{u,v};\dul{B}_v) = 0$ and
  \begin{align*}
    I(\dul{A}_u;\dul{B}_v|\dul{M}_{u,v})
    &= \frac{n-1}{n} I(\dul{A}_u|(\dul{M}_{u,v} = 0);\dul{B}_v|(\dul{M}_{u,v} = 0))
     + \frac{1}{n} I(\dul{A}_u|(\dul{M}_{u,v} = 1);\dul{B}_v|(\dul{M}_{u,v} = 1))\\
    &= \frac{n-1}{n}\cdot 0 + \frac{1}{n}I_{\ul{X}\ul{Y}}.
  \end{align*}
  Finally $I(\dul{A}_u;\dul{B}_v) \geq 0$ and $H(\dul{M}_{u,v}) = \frac{1}{n} \log n + \frac{n-1}{n}\log \frac{n}{n-1} \geq \frac{\log n}{n}$.
\end{proof}
\begin{proof}[Proof of Theorem \ref{thm:conLRT}]
Let \(\hat{\dul{M}}_{u,v}\define \ind{\bpth{\dul{A}_u,\dul{B}_v}\in H_\tau}\) denote the estimation on the relation between identifiers \(u\) and \(v\). We have a correct estimation if \(\hat{\dul{M}}_{u,v} = \dul{M}_{u,v}\). 
Define \(E \define \ind{\hat{\dul{M}}_{u,v} \neq \dul{M}_{u,v}}\). Then by Fano's inequality,
\begin{align*}
    H(\dul{M}_{u,v}|\dul{A}_u,\dul{B}_v) \leq H(E) + \Pr[E=1],
\end{align*}
which gives the upper bound as \(H(E)\).

Let \(\epsilon \define \Pr[E=1]\). This value can also be expressed as the expected frequency of false matches, i.e. given \(\varepsilon_{FN}\) and \(\varepsilon_{FP}\) the expected number of false negatives and false positives, \(\epsilon = \frac{\varepsilon_{FN}+\varepsilon_{FP}}{|\UA\times\UB|} = \frac{\varepsilon_{FN}+\varepsilon_{FP}}{n^2}\).
    
Let \(H_b\) denote the binary entropy function. By Fano's inequality, using Lemma \ref{lemma:conditionalEntropy}, we have
\begin{equation}
    H_b(\epsilon) \geq H(\dul{M}_{u,v}|\dul{A}_u,\dul{B}_v) \geq \frac{\log n-I_{\ul{X}\ul{Y}}}{n} \label{eqn:conLRT3}
\end{equation}
We have
\begin{align*}
    H_b(\epsilon) &\leq -\epsilon\log\epsilon + \epsilon
    = \frac{\varepsilon_{FN}+\varepsilon_{FP}}{n^2}\pth{2\log n - \log \pth{\varepsilon_{FN}+\varepsilon_{FP}} + 1}.
\end{align*}
Combining this with (\ref{eqn:conLRT3}) gives us
\begin{equation*}
    \varepsilon_{FN}+\varepsilon_{FP} \geq \frac{n}{2} \pth{\frac{\log n -I_{\ul{X}\ul{Y}}}{2 \log n + 1}}
\end{equation*}
and the claim follows.
\end{proof}

\bcomment{
\newtheorem*{tConLRT}{Theorem \ref*{thm:conLRT}}
\begin{tConLRT}
  \thmConLRT
\end{tConLRT}

\begin{proof}
    Define \(E \define \ind{Z\neq \hat{Z}}\). Then by Fano's inequality
    \begin{align*}
        H(Z|F) \leq H(E) + \Pr[E=1]\log(|\{0,1\}|-1) = H(E)
    \end{align*}
    Let \(\epsilon \define \Pr[E=1]\). This value can also be expressed as the expected frequency of false matches, i.e. given \(\varepsilon_{FN}\) and \(\varepsilon_{FP}\) the expected number of false negatives and false positives, \(\epsilon = \frac{\varepsilon_{FN}+\varepsilon_{FP}}{|\UA\times\UB|} = \frac{\varepsilon_{FN}+\varepsilon_{FP}}{n^2}\).
    
    Let \(H_b\) denote the binary entropy function. By Fano's inequality, using Lemma \ref{lemma:conditionalEntropy}, we have
    \begin{align}
        H_b(\epsilon) &\geq H(Z|F) = H(Z|\ul{X},\ul{Y}) \nonumber\\
        &\geq \frac{\log n-I_{\ul{X}\ul{Y}}}{n} + o(1) \label{eqn:conLRT3}
    \end{align}
    
    For \(\epsilon = o(1)\) we have
    \begin{align*}
        H_b(\epsilon) &= -\epsilon\log\epsilon + o(1) \nonumber\\
        &= \frac{\varepsilon_{FN}+\varepsilon_{FP}}{n^2}\pth{2\log n - \log \pth{\varepsilon_{FN}+\varepsilon_{FP}}} + o(1) \nonumber\\
        &\leq \frac{2\pth{\varepsilon_{FN}+\varepsilon_{FP}}\log n}{n^2}.
    \end{align*}
    Combining this with (\ref{eqn:conLRT3}) gives us
    \begin{align*}
        \frac{I_{\ul{X}\ul{Y}}}{\log n} \geq 1-\frac{2\pth{\varepsilon_{FN}+\varepsilon_{FP}}}{n}
    \end{align*}
\end{proof}
}

\appendix
\appendix

\section{Transformation of feature vectors}
\label{sec:featureTransformation}

By the invertibility of \(\Siga\) and \(\Sigb\), there exist Cholesky decompositions with invertible matrices \(\La\) and \(\Lb\) such that \(\Siga=\La\La^\top\) and \(\Sigb = \Lb\Lb^\top\).

Furthermore let \(\dul{U}\Sigab'\dul{V}^\top = \La^{-1}\Sigab\bpth{\Lb^\top}^{-1}\) be a singular-value decomposition with \(\dul{U}\) and \(\dul{V}\) unitary matrices and \(\Sigab'\) a diagonal matrix of size \(d_1\times d_2\).

Finally let \(\mua\) and \(\mub\) denote the mean of features from the two datasets.

Let the feature transformations be defined as
\begin{align*}
    \ul{X}' = \Ta(\ul{X}) &= \bpth{\dul{U}^\top\La^{-1}}\bpth{\ul{X}-\mua}\\
    \ul{Y}' = \Tb(\ul{Y}) &= \bpth{\dul{V}^\top\Lb^{-1}}\bpth{\ul{Y}-\mub}.
\end{align*}

It can be verified that
\begin{align*}
    \mathbb{E}\crocVec{\ul{X}'}{\ul{Y}'} = \ul{0} \textrm{ and }\dul{\Sigma}'\define \E{\crocVec{\ul{X}'}{\ul{Y}'}\crocVec{\ul{X}'}{\ul{Y}'}^\top} = \crocMat{\dul{I}^{d_1}}{\Sigab'}{\Sigab'^\top}{\dul{I}^{d_2}}.
\end{align*}

Both \(\dul{U}^\top\La^{-1}\) and \(\dul{V}^\top\Lb^{-1}\) are invertible matrices, therefore \(\Ta\) and \(\Tb\) are bijective functions with no loss of information.

Next we perform a non-reversible operation. Consider the zero rows and columns of \(\Sigab'\). These correspond to transformed feature entries that have no correlation with features from the other dataset, hence they provide no information in identifying matching features. Simply dropping these results in no loss of information. \(\Sigab'\) is a diagonal matrix, therefore the number of features we keep from either dataset is simply equal to the number of non-zero entries in \(\Sigab'\). Thus the final feature vectors are  of equal size \(d\). Let \(\ul{X}''\) and \(\ul{Y}''\) be the final features obtained after throwing away entries in \(\ul{X}'\) and \(\ul{Y}'\), and \(\Sigab''\in\realS^{d\times d}\) be the matrix obtained by removing zero rows and columns from \(\Sigab'\). Then,
\begin{align*}
    \dul{\Sigma}'' \define \E{\crocVec{\ul{X}''}{\ul{Y}''}\crocVec{\ul{X}''}{\ul{Y}''}^\top} = \crocMat{\dul{I}^d}{\Sigab'}{\Sigab'^\top}{\dul{I}^d}.
\end{align*}

Let \(\ul{\rho}\) denote the diagonal entries in \(\Sigab''\).

\begin{lemma}
    \label{lemma:mutualInformation}
    Let
    \begin{align*}
        \dul{\Sigma} = \crocMat{\Siga}{\Sigab}{\Sigab^\top}{\Sigb} \textrm{ and } \dul{\Sigma}'' = \crocMat{\dul{I}}{\diag(\ul{\rho})}{\diag(\ul{\rho})}{\dul{I}}
    \end{align*}
    be the covariance matrices of the original features and of the transformed features respectively. The mutual information between matched pairs of original features \((\ul{X},\ul{Y})\) and matched pairs of transformed features \((\ul{X}'',\ul{Y}'')\) is the same and given by
    \begin{align*}
        I_{XY} &= -\frac{1}{2}\sum_{i\in[d]}\log(1-\rho_i^2)\\
        &= -\frac{1}{2}\log \frac{\det\bpth{\dul{\Sigma}}}{\det\bpth{\Siga}\cdot\det\bpth{\Sigb}}.
    \end{align*}
\end{lemma}
\begin{proof}
    Our transformation keeps all information from features that exhibits correlation with the other feature list. Therefore the mutual information between feature pairs is the same for original features and transformed features.
    
    The final feature vector pairs have mutual information \(-\frac{1}{2}\log(1-\rho_i^2)\) per feature \(i\). Thus we have \(I_{XY} = -\frac{1}{2}\sum \log(1-\rho_i^2) = -\frac{1}{2}\log\det\bpth{\dul{\Sigma}''}\). It can be shown that removing the zero rows/columns in \(\dul{\Sigma}'\), which corresponds to removing rows/columns with all zeros except 1 on the diagonal in \(\dul{\Sigma}'\), does not change the determinant of the matrix. Then \(I_{XY} = \det\bpth{\dul{\Sigma}'}\). Finally we show that \(\det\bpth{\dul{\Sigma}'} = \frac{\det\bpth{\dul{\Sigma}}}{\det\bpth{\Siga}\cdot\det\bpth{\Sigb}}\).
    \begin{align*}
        \textrm{We have } \quad \det\bpth{\dul{\Sigma}'} = \det\bpth{\dul{I}-\Sigab'\Sigab'^\top} \quad \textrm{ and }
    \end{align*}
    \vspace{-0.5cm}
    \begin{align*}
        \bpth{\La\dul{U}}\bpth{\dul{I}-\Sigab'\Sigab'^\top}\bpth{\La\dul{U}}^\top &= \Siga - \Sigab\Sigb\Sigab^\top.\\
        \implies \enspace \textrm{det}^2\bpth{\La\dul{U}}\det\bpth{\dul{\Sigma}'} &= \det\bpth{\Siga - \Sigab\Sigb\Sigab^\top}
    \end{align*}
    Furthermore \(\det^2\bpth{\La\dul{U}} = \det\bpth{\La\dul{U}\dul{U}^\top\La^\top} = \det\bpth{\La\La^\top} = \det\bpth{\Siga}\). So
    \begin{align*}
        \det\bpth{\Siga}\det\bpth{\Sigb}\det\bpth{\dul{\Sigma}'} &= \det\bpth{\Sigb} \det\bpth{\Siga - \Sigab\Sigb\Sigab^\top}.
    \end{align*}
    Notice that the right-hand side is specifically the block matrix expression for the determinant of \(\dul{\Sigma}\). So \begin{align*}
        \det\bpth{\dul{\Sigma}'} &=\frac{\det\bpth{\dul{\Sigma}}}{\det\bpth{\Siga}\cdot\det\bpth{\Sigb}}\\
        &= -\frac{1}{2}\sum_{i\in[d]}\log(1-\rho_i^2).
    \end{align*}
\end{proof}

\begin{lemma}
    \label{lemma:variance}
    Let \(p_X\), \(p_Y\) and \(p_{XY}\) denote the probability density functions of \(\ul{X}\), \(\ul{Y}\) and \((\ul{X},\ul{Y})\) respectively, where \((\ul{X},\ul{Y})\) are a pair of correlated feature vectors. Let
    \begin{align*}
        \dul{\Sigma} = \crocMat{\Siga}{\Sigab}{\Sigab^\top}{\Sigb} \textrm{ and } \dul{\Sigma}'' = \crocMat{\dul{I}}{\diag(\ul{\rho})}{\diag(\ul{\rho})}{\dul{I}}
    \end{align*}
    be the covariance matrices of these features and of their transformed variants, as described above. Then
    \begin{align*}
        \Var\pth{\log\frac{p_{XY}(\ul{X},\ul{Y})}{p_X(\ul{X})p_Y(\ul{Y})}} &= \sum_{i\in[d]}\rho_i^2\\
        &= \tr\bpth{\Siga^{-1}\Sigab\Sigb^{-1}\Sigab^\top}
    \end{align*}
\end{lemma}
\begin{proof}
    Notice that the log likelihood ratio of the transformed features are the same as the log likelihood ratio for the original features. Thus we just assume that the original features vectors already have the covariance matrix of the specified form. We already know the mean of the log likelihood ratio, which is equal to the mutual information \(I_{XY}\). Then
    \begin{align*}
        \frac{p_{XY}(\ul{X},\ul{Y})}{p_X(\ul{X})p_Y(\ul{Y})} - I_{XY} &= \sum_{i\in[d]}-\frac{\rho_i^2(X_i^2+Y_i^2)-2\rho_i^2X_iY_i}{2(1-\rho_i^2)}
    \end{align*}
    Let \(Z_i \define \rho_i^2\pth{X_i^2+Y_i^2}-2\rho_iX_iY_i\) for any \(i\in[d]\).
    Then
    \begin{align*}
        \Var&\pth{\frac{p_{XY}(\ul{X},\ul{Y})}{p_X(\ul{X})p_Y(\ul{Y})}} = \E{\frac{Z_i^2}{4(1-\rho_i^2)^2}}\\
        &= \sum_{i\in[d]} \frac{\E{Z_i^2}}{4\pth{1-\rho_i^2}^2} + \sum_{\{i,j\}\in\binom{[d]}{2}} \frac{2\E{Z_i}\E{Z_j}}{4\pth{1-\rho_i^2}\pth{1-\rho_j^2}}
    \end{align*}
    Observe that, for any \(i\in[d]\)
    \begin{align*}
        \E{Z_i} = \rho_i^2\E{X_i^2}+\rho_i^2\E{Y_i^2} - 2\rho_i\E{X_iY_i} = 0.
    \end{align*}
    Furthermore
    \begin{align*}
        \E{Z_i^2} =& \rho_i^2\pth{\E{X_i^4}+\E{Y_i^4}}\\
        &- 4\rho_i^3\pth{\E{X_i^3Y_i}+\E{X_iY_i^3}}\\
        &+ \pth{2\rho_i^2+4\rho_i^2}\E{X_i^2Y_i^2}
    \end{align*}
    Note that \((X_i,Y_i)\) has the same distribution as \((Y_i,X_i)\). Then the expression for \(\E{Z_i^2}\) simplifies and we get
    \begin{align*}
        \E{Z_i^2} =&\,\, 2\rho_i^2\E{X_i^4} - 8\rho_i^3\E{X_i^3Y_i}\\
        &+ \pth{2\rho_i^4+4\rho_i^2}\E{X_i^2Y_i^2}.\\
        \intertext{By Isserli's theorem this gives us}
        \E{Z_i^2} =&\,\, 6\rho_i^4\mathbb{E}^2\croc{X_i^2}-24\rho^3\E{X_i^2}\E{X_iY_i}\\
        &+\pth{2\rho_i^4+4\rho_i^2}\cdot\pth{\E{X_i^2}\E{Y_i^2}+2\mathbb{E}^2\croc{X_iY_i}}\\
        &= 4\rho_i^2 - 8\rho_i^4 + 4\rho_i^6 = 4\rho_i^2\pth{1-\rho_i^2}^2.
    \end{align*}
    Then \(\frac{\E{Z_i^2}}{4\pth{1-\rho_i^2}^2} = \rho_i^2\) and \(\Var\pth{\log \frac{q(\alpha,\beta)}{q(\alpha)q(\beta)}} = \sum \rho_i^2\).
    
    Finally we derive the expression for this variance based on the original covariance matrix. We refer to matrices \(\Sigab''\), \(\Sigab'\), \(\La\), \(\dul{U}\) as given in the description of the feature transformation.
    \begin{align*}
        \sum_{i\in[d]} \rho_i^2 &= \tr\bpth{\Sigab''\Sigab''^\top}
        = \tr\bpth{\Sigab'\Sigab'^\top}\\
        &= \tr\Bpth{\bpth{\dul{U}^\top\La^{-1}\Sigab\bpth{\Lb^{-1}}^\top V}\bpth{\dul{V}^\top\Lb^{-1}\Sigab^\top\bpth{\La^{-1}}^\top U}}\\
        &= \tr\Bpth{\bpth{\La^{-1}}^\top\La^{-1}\Sigab\bpth{\Lb^{-1}}^\top\Lb^{-1}\Sigab^\top}\\
        &= \tr\bpth{\Siga^{-1}\Sigab\Sigb^{-1}\Sigab^\top}
    \end{align*}
\end{proof}

\section{Additional lemmas}

\begin{lemma}
    \label{lemma:elementary}
    Given any \(k\in\natS\), \(\mu\in\realS_+\) \(\delta\in[0,1]\) and \(z_1,z_2\cdots,z_{2k}\in[1-\delta,1+\delta]\) a sequence with mean \(\mu\), we have
    \begin{align*}
        \prod_{i\in[2s]} z_i \,\, \geq \pth{\mu-\delta}^k\pth{\mu+\delta}^k.
    \end{align*}
\end{lemma}
\begin{proof}
    We give an algorithmic proof.
    
    Consider a sequence \(z_1,z_2,\cdots,z_{2k}\) with mean \(\mu\) and all entries in the range \([\mu-\delta,\mu+\delta]\). Assume the product \(\pi\) of the sequence is not \(\pth{\mu-\delta}^k\pth{\mu+\delta}^k\). Then there exists two entries \(z_i,z_j\in(\mu-\delta,\mu+\delta)\). Without loss of generality assume \(|z_i-\mu|\geq|z_j-\mu|\) and \(z_i\leq z_j\). Modify the sequence by replacing \(z_i\) with \(\mu-\delta\) and \(z_j\) with \(z_j+z_i-\mu+\delta\). The new sequence still has mean \(\mu\) and entries within the same range. However observe that its product \(\frac{(\mu-\delta)(z_j+z_i-\mu+\delta)}{z_iz_j}\times \pi\) is strictly smaller.
    
    Iteratively applying this modification on any initial sequence eventually (in at most 2k-1 modifications) turns the sequence into one with product \(\pth{\mu-\delta}^k\pth{\mu+\delta}^k\), which cannot be larger than the product of any intermediary sequence or the initial sequence.
\end{proof}

\bibliographystyle{ieeetr}
\bibliography{bibtex.bib,deanon}

\end{document}